\documentclass{opt2025} % Include author names

\title[
EDIT: Early Diffusion Inference Termination for dLLMs Based on Dynamics of Training Gradients
]{
EDIT: Early Diffusion Inference Termination for dLLMs\\
Based on Dynamics of Training Gradients
}

\optauthor{%
\Name{Sebastian U. Stich} \Email{stich@cispa.de}\\
\Name{Author Name} \Email{xyz@sample.com}\\
\addr CISPA, Germany}

\optauthor{%
\Name{He-Yen Hsieh} \Email{heyenhsieh@g.harvard.edu}\\ \addr Harvard University \\
\Name{Hong Wang} \Email{hong.wang@intel.com} \\ \addr Intel Corporation\\
\Name{H.T. Kung} \Email{kung@harvard.edu} \\ \addr Harvard University
}
\usepackage[dvipsnames]{xcolor}

\definecolor{Black}{rgb}{0.0, 0.0, 0.0}
\definecolor{DarkGreen}{rgb}{0.10, 0.55, 0.10}
\definecolor{DeepSkyBlue3}{rgb}{0.0, 0.686, 0.843}
\definecolor{DarkTurquoise}{rgb}{0.0, 0.843, 0.843}
\definecolor{Cyan3}{rgb}{0.0, 0.843, 0.686}
\definecolor{LightSeaGreen}{rgb}{0.0, 0.686, 0.686}
\definecolor{RoyalBlue}{rgb}{0.20, 0.60, 0.86}
\definecolor{DeepSkyBlue}{rgb}{0.0, 0.686, 1.}
\definecolor{DodgerBlue}{rgb}{0.0, 0.529, 1.}
\definecolor{DodgerBlue2}{rgb}{0.0, 0.3725, 1.}
\definecolor{DodgerBlue3}{rgb}{0.0, 0.3725, 0.843}
\definecolor{DarkCyan}{rgb}{0.0, 0.54, 0.54}
\definecolor{Gray}{gray}{0.9}
\definecolor{ChromeYellow}{rgb}{1.0, 0.65, 0.0}
\definecolor{Gold}{rgb}{1.0, 0.843, 0.0}
\definecolor{Crimson}{rgb}{0.86, 0.08, 0.24}
\definecolor{IndianRed}{rgb}{1.0, 0.373, 0.529}
\definecolor{SunsetOrange}{rgb}{0.98, 0.37, 0.33}
\definecolor{DarkOrange}{rgb}{1.0, 0.529, 0.}
\definecolor{DarkYellow}{rgb}{0.9,0.7,0.0}

\newcommand{\ie}{\textit{i.e.}}

\ifx \submission \undefined

\else
\fi

\usepackage{amsmath}
\usepackage{amssymb}
\usepackage{booktabs}
\usepackage[table]{xcolor}
\usepackage{multirow}
\usepackage{diagbox}
\usepackage{fontawesome}
\usepackage{pifont}
\usepackage{subcaption}
\usepackage{bbm}
\usepackage{algorithm}
\usepackage{algorithmic}
\usepackage{wrapfig}
\usepackage{graphicx}

\newtheorem{assumption}[theorem]{Assumption}

\begin{document}

\maketitle

\vspace{-25pt}
\begin{abstract}
Diffusion-based large language models (dLLMs) refine token generations through iterative denoising, but answers often stabilize before all steps complete. We propose EDIT (Early Diffusion Inference Termination), an inference-time criterion that adaptively stops denoising once sufficient reasoning stability relative to training-time reasoning is detected. EDIT monitors the alignment between token activations and a reasoning map derived from AdamW-aggregated LoRA updates captured during supervised fine-tuning (SFT). During training, optimization dynamics generate rich metadata about parameter importance that in prior methods is typically discarded upon model release. We preserve this information as a compact representation of learned reasoning pathways. During inference, alignment scores are converted to a distribution over the tokens already unmasked at the current denoising step, and convergence is detected when KL divergence between consecutive steps falls below a threshold on the matched unmasked (visible) tokens.
Across reasoning benchmarks, EDIT reduces diffusion steps by 11.8\% to 68.3\% while preserving or improving accuracy in most settings, with approximately 0.02\% storage overhead (about 1.5-2 MB for all QKV modules across 32 blocks in an 8 GB model).
By utilizing training-gradient dynamics, our work opens a new research direction for reducing dLLM inference time and cost. \href{https://github.com/louisYen/EDIT}{Code Repository}
\end{abstract}

\section{Introduction}
\label{sec:introduction}

Modern language model deployment follows a wasteful paradigm: during training, optimization dynamics on gradients generate rich information about which parameters are critical for specific capabilities, yet this metadata is routinely discarded once training completes. We challenge this practice by demonstrating that training-time optimization trajectories contain valuable signals that can guide intelligent inference-time decisions, specifically enabling adaptive early inference termination for diffusion language models.

Diffusion-based language models (dLLMs) \cite{NieZYZOHZLWL25, ZhaoGZG25, ArriolaGCYQHSK25} represent a promising alternative to autoregressive generation of tokens, employing iterative denoising processes that progressively refine outputs. However, their inference remains computationally expensive due to the fixed number of denoising steps, even when high-quality outputs emerge early in the process. Current approaches to this challenge operate without knowledge of which model parameters drove learning during training, essentially making uninformed decisions during inference termination optimization.

We introduce EDIT (Early Diffusion Inference Termination), a method that utilizes training optimization metadata to identify opportunities for early inference termination. Our key insight is that the AdamW optimizer's \cite{LoshchilovH19} moment estimates during fine-tuning training encode which parameters consistently receive strong, directionally-aligned updates when learning reasoning tasks.
The effectiveness of EDIT is validated from the perspective of gradient convergence during inference, as shown in Figure \ref{fig:grad_analysis}.
% \note{later on Fig 1, we use gradient view to discuss termination to find interpretation on the reason why this temrination criteria makes sense}.
These patterns, which we term \emph{AdamW evolution}, represent a map of the model's learned reasoning pathways. Rather than discarding this information when training is complete, we store it as compact metadata (requiring minimal additional storage) and use it to guide inference termination decisions.

\paragraph{Our Approach: Utilizing Training Metadata for Guiding Inference Termination.}
During supervised fine-tuning (SFT) on reasoning tasks, certain LoRA parameters receive consistent gradient signals across training steps, indicating their importance for encoding reasoning patterns. The AdamW optimizer naturally tracks this through its moment estimates, with the first moment capturing gradient direction and the second moment reflecting gradient stability. Parameters with large, stable updates become critical components of the learned reasoning pathways.
Traditional inference deployment discards this valuable information. EDIT preserves it by saving aggregated AdamW updates across training steps, creating a fingerprint of which parameters matter most for reasoning. At inference time, we compare current token activations against these preserved patterns using cosine similarity, assessing whether the model's current state aligns with its learned reasoning pathways. When this alignment stabilizes---indicating the model has reached its learned reasoning configuration---we can confidently terminate the diffusion inference process early.
Instead of relying on inference-time heuristics such as confidence or output stability, EDIT leverages training-time knowledge to terminate once key reasoning components are engaged.

\paragraph{Contributions.}
(1) We establish a new paradigm for early inference termination that leverages training metadata which is usually discarded in prior methods. (This approach opens future research directions in not only early inference termination, but also informing dynamic compute allocation, quality prediction, and other inference-time optimizations.)
(2) We provide a practical instantiation through EDIT, demonstrating that AdamW evolution patterns can reliably indicate when diffusion models have completed their core reasoning. Our method requires no architectural changes, adds minimal storage overhead, and integrates seamlessly with existing diffusion language models.
(3) We validate our approach across multiple reasoning benchmarks, showing inference speedups of 11.8\% to 68.3\% while maintaining or improving accuracy in most settings. These gains come purely from utilizing training information that already existed but was previously thrown away, highlighting the inefficiency of current practices.

\section{Preserving and Utilizing Training Metadata}
\label{sec:method}

We detail how EDIT captures optimization dynamics during training and leverages them for intelligent early termination during inference. Our approach consists of two phases: metadata extraction during fine-tuning (Section \ref{subsec:adamw_evolution}) and metadata-guided termination during inference (Section \ref{subsec:early_termination}).

\paragraph{Utilizing Training Metadata for Guiding Inference Termination.}
During supervised fine-tuning (SFT), some parameters receive strong, stable updates that encode core reasoning patterns, while others show weak or oscillating updates and contribute less. We track this distinction through the AdamW update history---what we call the \textit{AdamW evolution}---which forms a map of reasoning-relevant parameters. At inference, activations are compared against this map; when alignment with the learned reasoning pathways is reached, early termination is enabled with confidence.

\subsection{Capturing AdamW Evolution During Training}
\label{subsec:adamw_evolution}

We consider a pre-trained base model with LoRA (Low-Rank Adaptation) \cite{HuSWALWWC22} modules inserted into the Query ($Q$), Key ($K$), and Value ($V$) projections of each Transformer block. During SFT on reasoning tasks, only these LoRA parameters are updated. Each LoRA module consists of matrices $(A, B)$ where $A \in \mathbb{R}^{r \times d_{\text{in}}}$ and $B \in \mathbb{R}^{d_{\text{out}} \times r}$ implement a low-rank update to the corresponding projection.

\noindent \textbf{Notation note:} We use $\mathcal{L}_k$ for  the loss at training step $k$, and $L$ for block length in the diffusion process.

For the LoRA-B matrix $B \in \mathbb{R}^{d_{out} \times r}$ (where $d_{out}$ is the output dimension and $r$ is the rank), the AdamW optimizer maintains first and second moment estimates at each training step $k$:
\begin{equation}
\label{eq:adamw_moments}
M_{k,B} = \beta_1 M_{k-1,B} + (1-\beta_1) G_{k,B}, \quad
V_{k,B} = \beta_2 V_{k-1,B} + (1-\beta_2) G_{k,B}^{\odot 2},
\end{equation}
where $G_{k,B} = \nabla_{B} \mathcal{L}_k$ is the gradient tensor, $\beta_1, \beta_2 \in [0,1)$ are decay rates, and $\odot$ denotes element-wise operations.
The element-wise update magnitude at step $k$ is:
\begin{equation}
\label{eq:adamw_update}
U_{k,B} = \frac{M_{k,B}}{\sqrt{V_{k,B}} + \epsilon},
\end{equation}
where $\epsilon$ ensures numerical stability and division is element-wise.
To create a stable representation of the parameter importance patterns, we define the \emph{AdamW evolution tensor} as the average over all $\mathcal{K}$ fine-tuning steps:
\begin{equation}
\label{eq:adamw_evolution_tensor}
\bar{U}_B = \frac{1}{\mathcal{K}} \sum_{k=1}^{\mathcal{K}} U_{k,B} \in \mathbb{R}^{d_{out} \times r}.
\end{equation}

This tensor captures which elements of the LoRA-B matrix consistently received strong, directional updates during training. To enable comparison with token activations $\mathbf{f}_s \in \mathbb{R}^{d_{out}}$, we reduce $\bar{U}_B$ to a feature-aligned vector $\mathbf{u} \in \mathbb{R}^{d_{out}}$ using row-wise energy:
\begin{equation}
\label{eq:reduction}
\mathbf{u}[p] = \left\|\bar{U}_B[p,:]\right\|_2 = \sqrt{\sum_{j=1}^{r} \bar{U}_B[p,j]^2}, \quad p = 1, \ldots, d_{out}.
\end{equation}

This reduction preserves the update magnitude each output dimension receives across low-rank components, forming a parameter-importance signature aligned with the feature space.

\subsection{Metadata-Guided Early Termination During Inference}
\label{subsec:early_termination}

At inference time, we leverage the preserved AdamW evolution to determine when the diffusion process can safely terminate. Our approach operates on block-level diffusion \cite{NieZYZOHZLWL25, ZhaoGZG25, ArriolaGCYQHSK25}, where the sequence is divided into blocks of length $L$, and tokens within each block are progressively unmasked across denoising steps.

\subsubsection{Assessing Reasoning Alignment}
\label{subsubsec:alignment}

Let $\mathcal{S}_t$ denote the set of visible (unmasked) tokens at denoising step $t$. For each visible token $s \in \mathcal{S}_t$, we extract its post-LoRA activation $\mathbf{f}_s^{(t)} \in \mathbb{R}^{d_{out}}$ from the chosen module (specifically, the LoRA-B output of the Query projection in the last Transformer block, based on our empirical findings in Figure \ref{fig:parameter_activation}).
We compute the cosine similarity between each token's activation and the AdamW evolution vector:
\begin{equation}
\label{eq:similarity}
\text{Sim}_s^{(t)} = \frac{\langle \mathbf{f}_s^{(t)}, \mathbf{u} \rangle}{\|\mathbf{f}_s^{(t)}\|_2 \|\mathbf{u}\|_2},
\end{equation}
where $\mathbf{u}$ is the feature-aligned vector from Equation \ref{eq:reduction}.
To convert these alignment scores into a probability distribution, we apply softmax with a fixed temperature $\tau_{\text{blk}}$ within each block:
\begin{equation}
\label{eq:distribution}
P^{(t)}(s) = \frac{\exp(\text{Sim}_s^{(t)}/\tau_{\text{blk}})}{\sum_{i \in \mathcal{S}_t} \exp(\text{Sim}_i^{(t)}/\tau_{\text{blk}})}, \quad s \in \mathcal{S}_t.
\end{equation}

Keeping $\tau_{\text{blk}}$ fixed within a block ensures that distribution changes reflect genuine alignment shifts rather than temperature-induced artifacts.

\subsubsection{Detecting Reasoning Stability via Matched Support}
\label{subsubsec:stability}

As tokens are progressively unmasked, the support of our distribution grows. To properly compare distributions across steps, we must account for this changing support. Let $\mathcal{I}_t = \mathcal{S}_{t-1} \cap \mathcal{S}_t$ be the intersection of visible tokens between consecutive steps.
We renormalize both distributions to this common support:
\begin{equation}
\label{eq:renorm}
\tilde{P}^{(t)}(s) = \frac{P^{(t)}(s)}{\sum_{i \in \mathcal{I}_t} P^{(t)}(i)}, \quad \tilde{P}^{(t-1)}(s) = \frac{P^{(t-1)}(s)}{\sum_{i \in \mathcal{I}_t} P^{(t-1)}(i)}, \quad s \in \mathcal{I}_t.
\end{equation}

The step-wise divergence is then computed as:
\begin{equation}
\label{eq:kl_divergence}
D_t = D_{\text{KL}}(\tilde{P}^{(t)} \parallel \tilde{P}^{(t-1)}) = \sum_{s \in \mathcal{I}_t} \tilde{P}^{(t)}(s) \log \frac{\tilde{P}^{(t)}(s)}{\tilde{P}^{(t-1)}(s)}.
\end{equation}

\subsubsection{Early Termination with Consecutive Stability}
\label{subsubsec:termination}
% \paragraph{Early Termination with Consecutive Stability.}
To ensure robust termination decisions, we require consecutive steps of stability rather than sporadic stable steps.
We maintain a run-length counter $c$ updated as $c \!\leftarrow\! c+1$ if $D_t < \delta$, and reset to $0$ otherwise.

Beyond this heuristic, our analysis shows that small matched-support KL divergence over $\Omega$ consecutive steps, with $\Omega$ determined by the inference progress at the point of early termination, bounds the multi-step total variation distance and yields a no-flip condition for predicted tokens together with stability bounds for Lipschitz observables. PAC-style bounds (Corollary~\ref{cor:pac}) provide a basis for principled hyperparameter selection and indicate that early termination matches the full denoising during inference with high probability. Formal statements and calibration rules are given in Appendix~\ref{app:guarantees}, along with two extensions:
a token-wise freezing scheme that halts denoising per token with provable instance-wise safety (Appendix~\ref{app:token}), and a subspace generalization that replaces the single reasoning vector with a low-dimensional reasoning subspace while preserving all theoretical guarantees (Appendix~\ref{app:subspace}).

The diffusion process for the current block terminates when $c \geq \Omega$, indicating that the model's reasoning alignment has remained stable for $\Omega$ consecutive steps.
% Algorithm \ref{alg:edit} (Appendix~\ref{app:edit_algo}) presents the complete EDIT procedure.
Algorithm~\ref{alg:edit} and Figure~\ref{fig:EDIT_diagram} (Appendix~\ref{app:edit_algo}) present the full EDIT procedure and its workflow.
The training phase extracts metadata with zero additional computational cost (these values are already computed by the optimizer), while the inference phase uses this metadata to make principled termination decisions.

\section{Experimental Validation}
\label{sec:experiments}

\paragraph{Experimental Setup.}
We evaluate on five reasoning tasks: Countdown \cite{PZWYPS25}, Sudoku \cite{Sudoku}, MATH500 \cite{LightmanKBEBLLS24}, GSM8K \cite{CKBCJKPTHNHS21}, and GPQA \cite{Rein2023GPQA}. We use LLaDA-8B \cite{NieZYZOHZLWL25} as our baseline model, fine-tuned on the s1 dataset \cite{MYSLLHZLCH25} with LoRA applied to QKV projections. All experiments use Intel XPU hardware to ensure reproducibility.
During SFT, we preserve AdamW evolution metadata (Section~\ref{subsec:adamw_evolution}).
Persisting reduced vectors $\mathbf{u}$ requires $\sim$16 KB per module, or $\sim$1.5 MB for all QKV projections across 32 blocks ($<$0.02\% of an 8 GB model).
At inference, EDIT uses task-specific thresholds (Appendix~\ref{subsec:hyperparameters}) selected on held-out validation sets (20\% of training data), ensuring no test set leakage.

\paragraph{Results: Efficiency Gains with Preserved Accuracy.}
Table \ref{tab:accuracy} shows that EDIT improves accuracy on Countdown (up to 31.6\%) and Sudoku (up to 16.1\%), while remaining competitive on other tasks.
Early termination avoids late-step degradation because once predictions stabilize, further denoising can overwrite correct intermediate states. This effect is most pronounced in crisp tasks such as Countdown and Sudoku.
GSM8K at sequence length 512 drops from 81.2\% to 76.2\% because long reasoning chains often stabilize before reasoning is complete.
Although GSM8K and GPQA show task-specific variation, the overall average remains positive, validating that metadata-guided termination improves rather than compromises quality.
Table \ref{tab:efficiency} shows that EDIT reduces average denoising steps per block by 11.8\%–68.3\% compared to fixed 64/128/256-step baselines for sequence lengths 128/256/512.
Gains are most pronounced on short sequences where full diffusion is wasteful.
With PAC-style calibration (Appendix~\ref{subsec:app_calibration}), 72.3\% of early terminations satisfy Corollary~\ref{cor:pac}, indicating that EDIT remains within its theoretical safety bounds while delivering speedups.

\begin{table}[t]
\centering
\caption{
Accuracy on reasoning benchmarks. EDIT uses training-time metadata for adaptive early termination. Results are mean over 3 seeds, where \textbf{bold} denotes the best, \underline{underline} denotes the second-best. 0-shot means no in-context examples during evaluation (post-SFT). Experiments are run with sequence lengths 128/256/512.
}
\label{tab:accuracy}
\resizebox{0.9\textwidth}{!}{%
\begin{tabular}{@{}l|ccc|ccc|ccc|ccc|ccc@{}}
\toprule
\multirow{2}{*}{\diagbox[width=4.5cm]{Method}{Dataset (Seq Len)}} &
\multicolumn{3}{c|}{\textbf{Countdown (0-shot)}} &
\multicolumn{3}{c|}{\textbf{Sudoku (0-shot)}} &
\multicolumn{3}{c|}{\textbf{MATH500 (0-shot)}} &
\multicolumn{3}{c|}{\textbf{GSM8K (0-shot)}} &
\multicolumn{3}{c}{\textbf{GPQA (0-shot)}} \\
& 128 & 256 & 512 & 128 & 256 & 512 & 128 & 256 & 512 & 128 & 256 & 512 & 128 & 256 & 512 \\
\midrule
LLaDA (No SFT) & 
    \underline{19.9} & 19.5 & 16.4 & 
    10.4 & 6.4 & \underline{6.3} & 
    \textbf{27.6} & \underline{32.4} & 36.0 & 
    \underline{68.1} & 75.8 & \underline{79.5} & 
    21.9 & \textbf{27.9} & 25.7 \\
LLaDA (SFT) &
    19.5 & \underline{20.7} & \underline{20.3} & 
    \underline{11.4} & \underline{8.2} & 5.0 & 
    26.2 & 30.4 & \underline{35.4} & 
    \textbf{69.8} & \underline{77.0} & \textbf{81.2} & 
    \underline{23.0} & 20.5 & \textbf{26.3} \\
\textbf{EDIT (Ours)} &
    \textbf{28.9} & \textbf{31.6} & \textbf{27.7} & 
    \textbf{16.1} & \textbf{11.3} & \textbf{7.6} & 
    \underline{27.4} & \textbf{32.8} & \textbf{36.6} & 
    67.3 & \textbf{77.6} & 76.2 & 
    \textbf{25.5} & \underline{27.7} & \underline{26.1} \\
\bottomrule
\end{tabular}
}
\end{table}

\begin{table}[t]
\centering
\caption{
Diffusion steps with EDIT vs. baseline full diffusion.
Values are averaged across blocks, and baselines are fixed at 64/128/256 steps for sequence lengths 128/256/512.
Percentages show reduction from baseline steps, with training metadata enabling confident early termination without quality loss.
}
\label{tab:efficiency}
\resizebox{0.9\textwidth}{!}{%
\begin{tabular}{@{}l|ccc|ccc|ccc|ccc|ccc@{}}
\toprule
\multirow{2}{*}{\diagbox[width=4.5cm]{Method}{Dataset (Seq Len)}} &
\multicolumn{3}{c|}{\textbf{Countdown}} &
\multicolumn{3}{c|}{\textbf{Sudoku}} &
\multicolumn{3}{c|}{\textbf{MATH500}} &
\multicolumn{3}{c|}{\textbf{GSM8K}} &
\multicolumn{3}{c}{\textbf{GPQA}} \\
& 128 & 256 & 512 & 128 & 256 & 512 & 128 & 256 & 512 & 128 & 256 & 512 & 128 & 256 & 512 \\
\midrule
Baseline (Full Steps) & 64 & 128 & 256 & 64 & 128 & 256 & 64 & 128 & 256 & 64 & 128 & 256 & 64 & 128 & 256 \\
\textbf{EDIT (Ours)} & 
    40.4 & 40.6 & 133.3 & 
    38.3 & 74.9 & 163.3 & 
    38.1 & 81.9 & 197.2 & 
    42.8 & 103.5 & 225.8 & 
    40.3 & 81.3 & 194.1 \\
\textbf{Reduction (\%)} & 
    36.9 & 68.3 & 47.9 & 
    40.2 & 41.5 & 36.2 & 
    40.5 & 36.0 & 23.0 & 
    33.1 & 19.2 & 11.8 & 
    37.0 & 36.5 & 24.2 \\
\bottomrule
\end{tabular}
}
\end{table}

\paragraph{Gradient-Based Justification for Early Termination.}
\label{sec:gradview_edit}

To determine when denoising steps can be truncated safely, we adopt a gradient view of inference by comparing pseudo-gradients at inference with SFT gradients on LoRA-B layers. At each inference step $t$, the model outputs logits $z_t(s)$ for tokens $s$ in a block of length $L$. Since no ground-truth labels are available, the signal comes from changes in predictive distributions across steps. Let $p_\theta(z_t(s))$ and $p_\theta(z_{t+1}(s))$ denote predictions at steps $t$ and $t{+}1$; their KL divergence quantifies prediction change (see Appendix \ref{app:gradview_edit} for details).
We define the pseudo-gradient as $\tilde{G}_{t,B} = \nabla_{B} \sum_{s \in S_{t+1}}
\mathrm{KL}\!\left(p_\theta(z_t(s)) \,\|\, p_\theta(z_{t+1}(s))\right)$, where $S_{t+1}$ is the visible token set at step $t{+}1$.
Backpropagating this divergence through LoRA-B yields $\tilde{G}_{t,B}$, whose root mean square (RMS) magnitude provides a scalar summary per step. Tracing these values across denoising steps produces a trajectory of inference dynamics.

% \begin{wrapfigure}{r}{0.37\textwidth}
% \vspace{-15pt}
% \centering
% \includegraphics[width=\linewidth]{figures/grad_sft_vs_inference_transformerblk1.png}
% % \captionsetup{type=figure}
% \caption{
% Gradient-based analysis of training–inference alignment on GPQA (seq. 128, 2nd block). Root mean square (RMS) pseudo-gradients $\tilde{G}_{t,B}$ across steps are compared with the SFT gradient mean (dashed) and variance band (shaded). The convergence point (yellow \textcolor{DarkYellow}{$\blacktriangledown$}) occurs at step 19, after which pseudo-gradients stabilize near the SFT mean, indicating that $\sim$20 steps per block preserve fidelity while reducing computation (Table~\ref{tab:efficiency}, 40.3 steps for two blocks).
% }
% \label{fig:grad_analysis}
% \end{wrapfigure}

\begin{figure}[t]
% \vspace{-15pt}
\centering
\includegraphics[width=0.4\linewidth]{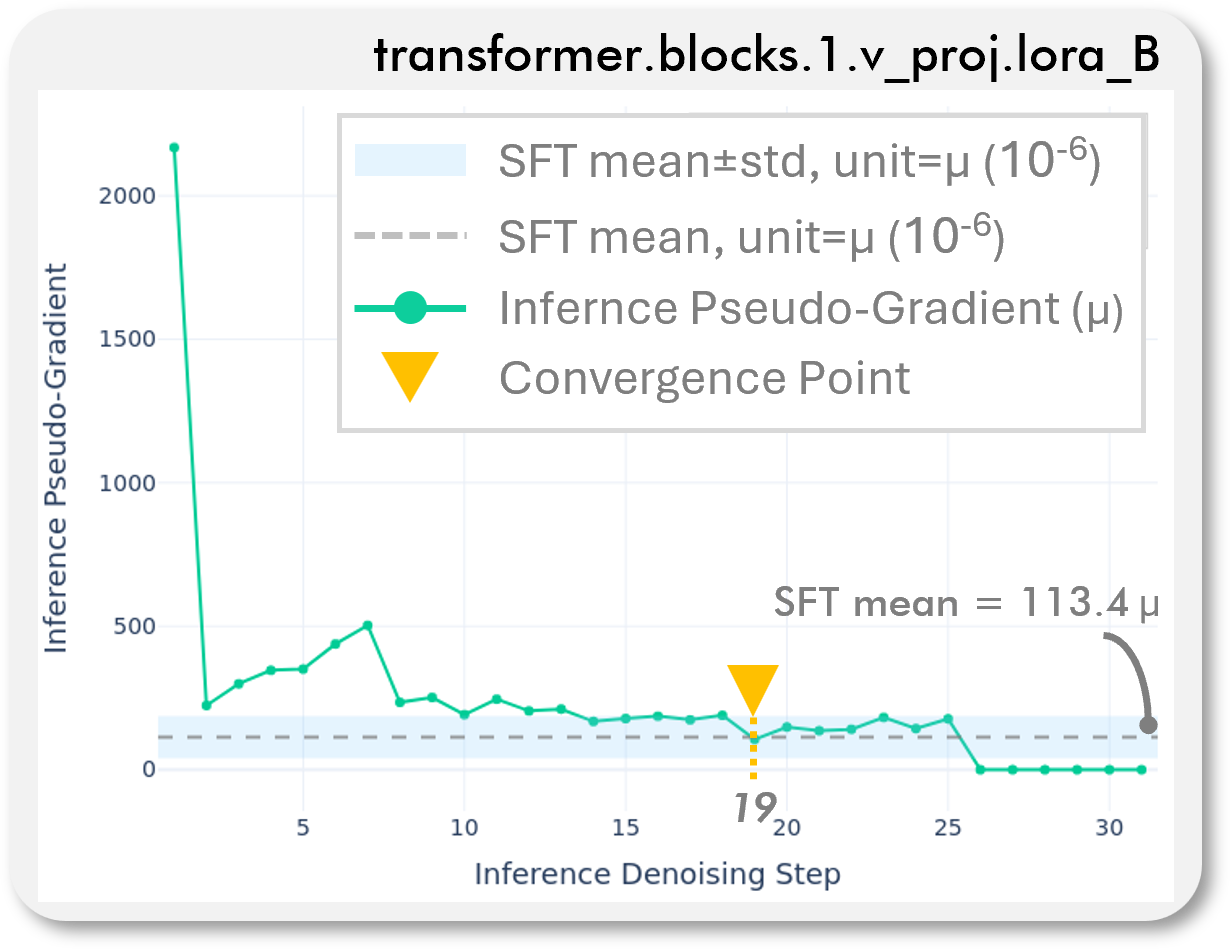}
% \captionsetup{type=figure}
\caption{
Gradient-based analysis of training–inference alignment on GPQA (seq. 128, 2nd block). Root mean square (RMS) pseudo-gradients $\tilde{G}_{t,B}$ across steps are compared with the SFT gradient mean (dashed) and variance band (shaded). The convergence point (yellow \textcolor{DarkYellow}{$\blacktriangledown$}) occurs at step 19, after which pseudo-gradients stabilize near the SFT mean, indicating that $\sim$20 steps per block preserve fidelity while reducing computation (Table~\ref{tab:efficiency}, 40.3 steps for two blocks).
}
\label{fig:grad_analysis}
\end{figure}

During SFT, we compute RMS magnitudes of gradients $G_{k,B}$ across training steps and summarize them by a mean $\mu_{\text{SFT}}$ and a variance band, defining the stable regime. Convergence is declared when inference pseudo-gradients (1) approach $\mu_{\text{SFT}}$ and (2) remain within this band, beyond this time further denoising adds cost without benefit.
On GPQA(seq. 128, 2nd block), Figure \ref{fig:grad_analysis} shows pseudo-gradients converging at step 19 (yellow $\color{DarkYellow}{\blacktriangledown}$). Thereafter they fluctuate around the SFT mean, indicating entry into the training-consistent regime. Terminating at $\sim$20 steps per block thus preserves fidelity while cutting cost, consistent with Table~\ref{tab:efficiency} (40.3 steps for two blocks) and Table~\ref{tab:accuracy}, which confirms accuracy remains competitive.

\paragraph{Storage and Computational Overhead.}
AdamW evolution metadata stores only the reduced vector $\mathbf{u} \in \mathbb{R}^{d_{out}}$ per LoRA module, not the full tensor $\bar{U}_B$. With $d_{out}=4096$ (float32), this is $\sim$16 KB per module, or $32 \times 3 \times 16$ KB $\approx$1.5 MB across all QKV projections in 32 blocks—just 0.02\% of an 8 GB model. At inference, EDIT adds cosine similarity (Equation~\ref{eq:similarity}) and KL divergence (Equation~\ref{eq:kl_divergence}) with cost $O(|\mathcal{S}_t| \cdot d_{out})$ and $O(|\mathcal{I}_t|)$, negligible compared to $O(L^2 \cdot d_{out})$ self-attention (Appendix~\ref{subsec:app_overhead}). Overall overhead is negligible, yielding net speedups.

\section{Conclusion and Future Directions}

We introduced EDIT, which preserves training metadata typically discarded and uses it to guide early inference termination in diffusion language models. By capturing optimization dynamics during fine-tuning, EDIT detects when reasoning is complete, reducing inference cost without architectural changes. Across five reasoning benchmarks, it achieves 11.8–68.3\% fewer diffusion steps while maintaining or improving accuracy, with only 0.02\% storage overhead.
EDIT has limitations: it requires training dynamics, often unavailable in released models, suggesting providers include optimization metadata; it depends on task-specific thresholds $(\delta,\Omega)$, motivating adaptive or learned criteria.
We evaluate only LoRA; full-parameter extensions remain future work.
Beyond early termination, training metadata could enable dynamic layer-wise compute, token-level processing, early quality prediction, and identifying prompts that merit additional inference budget.
This work also highlights a systemic inefficiency: training information is often discarded despite its value at inference.
Preserving and exploiting training metadata enables more holistic and efficient pipelines.

\bibliography{sample}

@inproceedings{LoshchilovH19,
  author       = {Ilya Loshchilov and
                  Frank Hutter},
  title        = {Decoupled Weight Decay Regularization},
  booktitle    = {7th International Conference on Learning Representations (ICLR)},
  year         = {2019}
}

@inproceedings{HuSWALWWC22,
  author       = {Edward J. Hu and
                  Yelong Shen and
                  Phillip Wallis and
                  Zeyuan Allen{-}Zhu and
                  Yuanzhi Li and
                  Shean Wang and
                  Lu Wang and
                  Weizhu Chen},
  title        = {LoRA: Low-Rank Adaptation of Large Language Models},
  booktitle    = {The Tenth International Conference on Learning Representations (ICLR)},
  year         = {2022}
}

@article{MYSLLHZLCH25,
  author       = {Niklas Muennighoff and
                  Zitong Yang and
                  Weijia Shi and
                  Xiang Lisa Li and
                  Li Fei{-}Fei and
                  Hannaneh Hajishirzi and
                  Luke Zettlemoyer and
                  Percy Liang and
                  Emmanuel J. Cand{\`{e}}s and
                  Tatsunori Hashimoto},
  title        = {s1: Simple test-time scaling},
  journal      = {arXiv preprint arXiv:2501.19393},
  year         = {2025}
}

@inproceedings{ArriolaGCYQHSK25,
  author       = {Marianne Arriola and
                  Aaron Gokaslan and
                  Justin T. Chiu and
                  Zhihan Yang and
                  Zhixuan Qi and
                  Jiaqi Han and
                  Subham Sekhar Sahoo and
                  Volodymyr Kuleshov},
  title        = {Block Diffusion: Interpolating Between Autoregressive and Diffusion
                  Language Models},
  booktitle    = {The Thirteenth International Conference on Learning Representations,
                  (ICLR)},
  year         = {2025}
}

@misc{Sudoku,
author       = {Arel},
title        = {Arel’s sudoku generator},
howpublished = {\url{https://www.ocf.berkeley.edu/~arel/sudoku/main.html}},
note         = {Accessed: 2025-04-08},
year         = {2025}
}

@misc{PZWYPS25,
author       = {Jiayi Pan and 
                Junjie Zhang and 
                Xingyao Wang and 
                Lifan Yuan and 
                Hao Peng and 
                Alane Suhr},
title        = {TinyZero},
howpublished = {https://github.com/Jiayi-Pan/TinyZero},
note         = {Accessed: 2025-01-24},
year         = {2025}
}

@inproceedings{LightmanKBEBLLS24,
  author       = {Hunter Lightman and
                  Vineet Kosaraju and
                  Yuri Burda and
                  Harrison Edwards and
                  Bowen Baker and
                  Teddy Lee and
                  Jan Leike and
                  John Schulman and
                  Ilya Sutskever and
                  Karl Cobbe},
  title        = {Let's Verify Step by Step},
  booktitle    = {The Twelfth International Conference on Learning Representations
                  (ICLR)},
  year         = {2024}
}

@article{CKBCJKPTHNHS21,
  author       = {Karl Cobbe and
                  Vineet Kosaraju and
                  Mohammad Bavarian and
                  Mark Chen and
                  Heewoo Jun and
                  Lukasz Kaiser and
                  Matthias Plappert and
                  Jerry Tworek and
                  Jacob Hilton and
                  Reiichiro Nakano and
                  Christopher Hesse and
                  John Schulman},
  title        = {Training Verifiers to Solve Math Word Problems},
  journal      = {arXiv preprint arXiv:2110.14168},
  year         = {2021}
}

@article{ZhaoGZG25,
  author       = {Siyan Zhao and
                  Devaansh Gupta and
                  Qinqing Zheng and
                  Aditya Grover},
  title        = {d1: Scaling Reasoning in Diffusion Large Language Models via Reinforcement
                  Learning},
  journal      = {arXiv preprint arXiv:2504.12216},
  year         = {2025}
}

@article{NieZYZOHZLWL25,
  author       = {Shen Nie and
                  Fengqi Zhu and
                  Zebin You and
                  Xiaolu Zhang and
                  Jingyang Ou and
                  Jun Hu and
                  Jun Zhou and
                  Yankai Lin and
                  Ji{-}Rong Wen and
                  Chongxuan Li},
  title        = {Large Language Diffusion Models},
  journal      = {arXiv preprint arXiv:2502.09992},
  year         = {2025}
}

@article{Rein2023GPQA,
  author       = {David Rein and
                  Betty Li Hou and
                  Asa Cooper Stickland and
                  Jackson Petty and
                  Richard Yuanzhe Pang and
                  Julien Dirani and
                  Julian Michael and
                  Samuel R. Bowman},
  title        = {{GPQA}: A Graduate-Level Google-Proof Q\&A Benchmark},
  journal      = {arXiv preprint arXiv:2311.12022},
  year         = {2023}
}

\newpage

\appendix
\section{Mathematical Details}
\label{app:mathematical}

\subsection{Complete Derivation of AdamW Evolution}
\label{subsec:app_derivation}

We provide the complete mathematical framework for extracting and utilizing AdamW evolution patterns. The key insight is that optimization dynamics during training create signatures of parameter importance that can guide inference.

For a LoRA-B matrix $B \in \mathbb{R}^{d \times r}$, the AdamW optimizer maintains exponentially weighted moving averages of gradients and squared gradients. Starting from the recursive definitions:
\begin{align}
M_{k,B}[i,j] &= \beta_1 M_{k-1,B}[i,j] + (1-\beta_1) G_{k,B}[i,j] \\
V_{k,B}[i,j] &= \beta_2 V_{k-1,B}[i,j] + (1-\beta_2) G_{k,B}[i,j]^2
\end{align}

By unrolling these recursions and assuming zero initialization, we obtain:
\begin{align}
M_{k,B}[i,j] &= (1-\beta_1) \sum_{\ell=1}^{k} \beta_1^{k-\ell} G_{\ell,B}[i,j] \\
V_{k,B}[i,j] &= (1-\beta_2) \sum_{\ell=1}^{k} \beta_2^{k-\ell} G_{\ell,B}[i,j]^2
\end{align}

The element-wise update magnitude captures both direction (through $M$) and reliability (through $V$):
\begin{equation}
U_{k,B}[i,j] = \frac{M_{k,B}[i,j]}{\sqrt{V_{k,B}[i,j]} + \epsilon}
\end{equation}

Averaging across all training steps yields the AdamW evolution tensor:
\begin{equation}
\bar{U}_B[i,j] = \frac{1}{\mathcal{K}} \sum_{k=1}^{\mathcal{K}} U_{k,B}[i,j]
\end{equation}

The reduction to feature space via row-wise energy (Equation \ref{eq:reduction}) preserves the total update magnitude each output dimension received, creating an interpretable signature of parameter importance.

\subsection{Theoretical Justification for Stability Detection}
\label{subsec:app_theory}

The KL divergence between consecutive alignment distributions provides a principled measure of reasoning stability. Under mild assumptions about the smoothness of the denoising process, we can show that stable KL divergence indicates convergence to a fixed point in the alignment space.

Consider the alignment distribution as a function of the denoising step: $P^{(t)} = f_t(\mathbf{x}, \mathbf{u})$ where $\mathbf{x}$ represents the current token states and $\mathbf{u}$ is the fixed AdamW evolution vector. If the denoising process is contractive in the alignment space (which can occur under conditions such as Lipschitz continuity of the denoiser combined with fixed-temperature softmax normalization), then:
\begin{equation}
D_{\text{KL}}(P^{(t+1)} \parallel P^{(t)}) \leq \gamma \cdot D_{\text{KL}}(P^{(t)} \parallel P^{(t-1)})
\end{equation}
for some $\gamma < 1$. This ensures that requiring $D_t < \delta$ for $\Omega$ consecutive steps provides strong evidence of convergence.

\section{Extended Experimental Details}
\label{app:experimental}

\subsection{Hyperparameter Selection Protocol}
\label{subsec:hyperparameters}

To ensure reproducibility and avoid overfitting, we employ a systematic hyperparameter selection protocol. For each task, we use 20\% of the training data as a validation set to tune the stability threshold $\delta \in \{0.025, 0.05, 0.1, 0.25, 0.45, 0.55\}$ and stability span $\Omega \in \{6, 8, 10, 12\}$. We select the configuration that maximizes the accuracy-efficiency trade-off, defined as accuracy divided by average diffusion steps. The block temperature $\tau_{\text{blk}}$ is fixed at 1.0 for all experiments to ensure fair comparison. Table \ref{tab:hyperparameters} provides the complete configuration for each benchmark.

\subsection{Complete Hyperparameter Configuration}
\label{subsec:app_hyperparameters}

Table \ref{tab:hyperparameters} provides the complete hyperparameter settings used in our experiments. These were selected using the validation protocol described in Section \ref{subsec:hyperparameters}.

\begin{table}[t]
\centering
\caption{EDIT hyperparameter configuration for each benchmark and sequence length. Parameters were selected on validation sets to optimize the accuracy-efficiency trade-off.}
\label{tab:hyperparameters}
\resizebox{\textwidth}{!}{%
\begin{tabular}{@{}l|ccc|ccc|ccc|ccc|ccc@{}}
\toprule
\multirow{2}{*}{\diagbox[width=4cm]{Parameter}{Dataset (Len)}} &
\multicolumn{3}{c|}{\textbf{Countdown}} &
\multicolumn{3}{c|}{\textbf{Sudoku}} &
\multicolumn{3}{c|}{\textbf{MATH500}} &
\multicolumn{3}{c|}{\textbf{GSM8K}} &
\multicolumn{3}{c}{\textbf{GPQA}} \\
& 128 & 256 & 512 & 128 & 256 & 512 & 128 & 256 & 512 & 128 & 256 & 512 & 128 & 256 & 512 \\
\midrule
\multicolumn{16}{l}{\textbf{All Blocks}} \\
Block Temperature ($\tau_{\text{blk}}$) & 1.0 & 1.0 & 1.0 & 1.0 & 1.0 & 1.0 & 1.0 & 1.0 & 1.0 & 1.0 & 1.0 & 1.0 & 1.0 & 1.0 & 1.0 \\
\midrule
\multicolumn{16}{l}{\textbf{First Block}} \\
Threshold ($\delta$) & 0.05 & 0.05 & 0.05 & 0.05 & 0.05 & 0.05 & 0.05 & 0.05 & 0.05 & 0.05 & 0.05 & 0.05 & 0.05 & 0.05 & 0.05 \\
Stability Span ($\Omega$) & 6 & 6 & 6 & 6 & 6 & 6 & 6 & 6 & 6 & 6 & 6 & 6 & 6 & 6 & 6 \\
\midrule
\multicolumn{16}{l}{\textbf{Subsequent Blocks}} \\
Threshold ($\delta$) & 0.05 & 0.55 & 0.45 & 0.1 & 0.45 & 0.05 & 0.1 & 0.05 & 0.025 & 0.05 & 0.025 & 0.025 & 0.05 & 0.05 & 0.05 \\
Stability Span ($\Omega$) & 6 & 12 & 12 & 6 & 12 & 6 & 6 & 6 & 6 & 6 & 6 & 6 & 6 & 6 & 8 \\
\bottomrule
\end{tabular}
}
\end{table}

\subsection{Understanding When Training Metadata Helps}
\label{subsec:analysis}

Figure \ref{fig:subdomain_analysis} reveals that EDIT's effectiveness varies across problem types within GPQA. The method shows substantial improvements in domains requiring systematic reasoning (Molecular Biology, Astrophysics) while providing modest gains in others. This variation supports our core thesis: the training metadata captures task-specific patterns, and its utility depends on how well-defined these patterns are for each domain. Tasks with clear, consistent reasoning pathways benefit most from our approach.

\begin{figure}[t]
\centering
\includegraphics[width=0.99\textwidth]{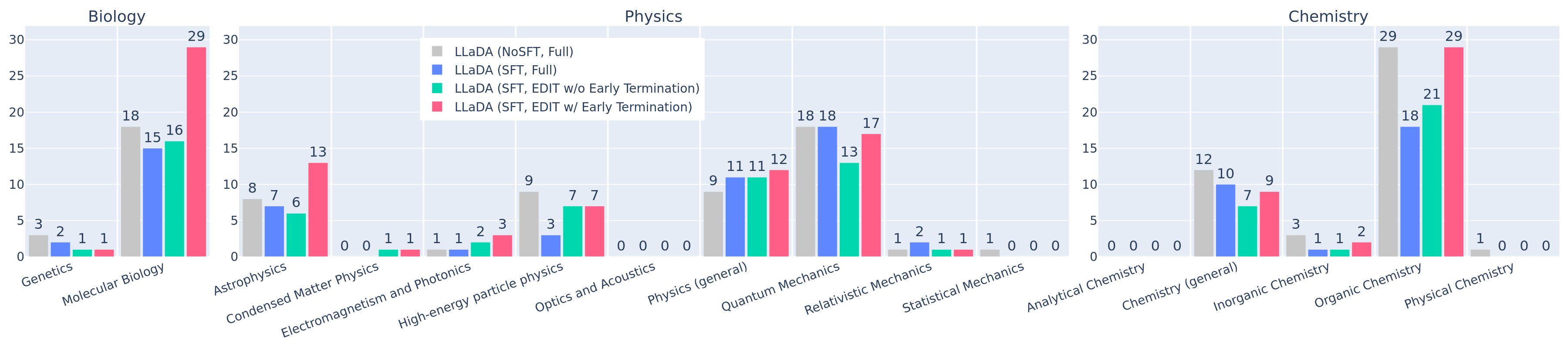}
\caption{Performance breakdown across GPQA subdomains comparing EDIT (red) with baseline SFT (green). EDIT shows particularly strong improvements in Molecular Biology and Astrophysics, where reasoning patterns are more structured. The domain-specific variation validates that training metadata captures specialized reasoning pathways.}
\label{fig:subdomain_analysis}
\end{figure}

Figure \ref{fig:parameter_activation} visualizes how different reasoning tasks activate distinct parameter subsets, as revealed by the AdamW evolution patterns. We focus on the LoRA-B matrix of the Query projection in the last Transformer block (block 31), which empirically shows the strongest task-specific patterns. This visualization confirms that training dynamics create meaningful signatures that can guide inference decisions.

\begin{figure}[t]
\centering
\includegraphics[width=0.99\textwidth]{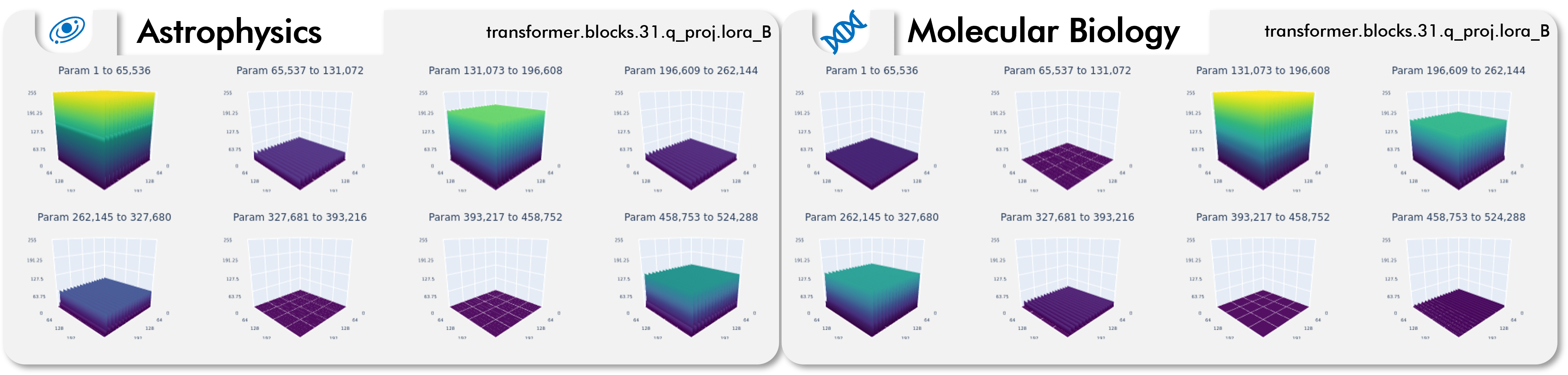}
\caption{Task-specific parameter activation patterns revealed by AdamW evolution. Different GPQA subdomains (Astrophysics vs. Molecular Biology) engage distinct parameter subsets in the LoRA-B matrix of the Query projection (transformer.block.31). The 3D visualization shows how parameter importance varies across tasks, demonstrating that training metadata captures specialized reasoning pathways.}
\label{fig:parameter_activation}
\end{figure}

\subsection{Storage and Computational Overhead}
\label{subsec:app_overhead}
The AdamW evolution metadata requires storing only the reduced vector $\mathbf{u} \in \mathbb{R}^d$ per chosen LoRA module, not the full tensor $\bar{U}_B$. For our configuration with $d = 4096$, this amounts to approximately 16 KB per module (assuming float32 precision). Even if we store metadata for all QKV projections across all 32 Transformer blocks, the total overhead is $32 \times 3 \times 16$ KB $\approx$ 1.5 MB—merely 0.02\% of the 8 GB model size.
At inference time, EDIT adds cosine similarity computations (Equation \ref{eq:similarity}) and KL divergence calculations (Equation \ref{eq:kl_divergence}) at each denoising step. These operations have complexity $O(|\mathcal{S}_t| \cdot d)$ and $O(|\mathcal{I}_t|)$ respectively, which is minimal compared to the $O(L^2 \cdot d)$ cost of self-attention in each Transformer block. The net result is substantial speedup despite these additional computations.

\subsection{Additional Visualizations}
\label{subsec:app_visualizations}

Figure \ref{fig:loraB_updates} provides additional evidence for the importance of preserving training-time metadata for LoRA-B. The visualization shows how specific parameters receive consistently strong updates during fine-tuning, creating clear signatures that can guide inference.
In contrast, Figure \ref{fig:loraA_updates} demonstrates that LoRA-A exhibits minimal parameter changes during the same training process.

\begin{figure}[!ht]
\centering
\includegraphics[width=0.99\textwidth]{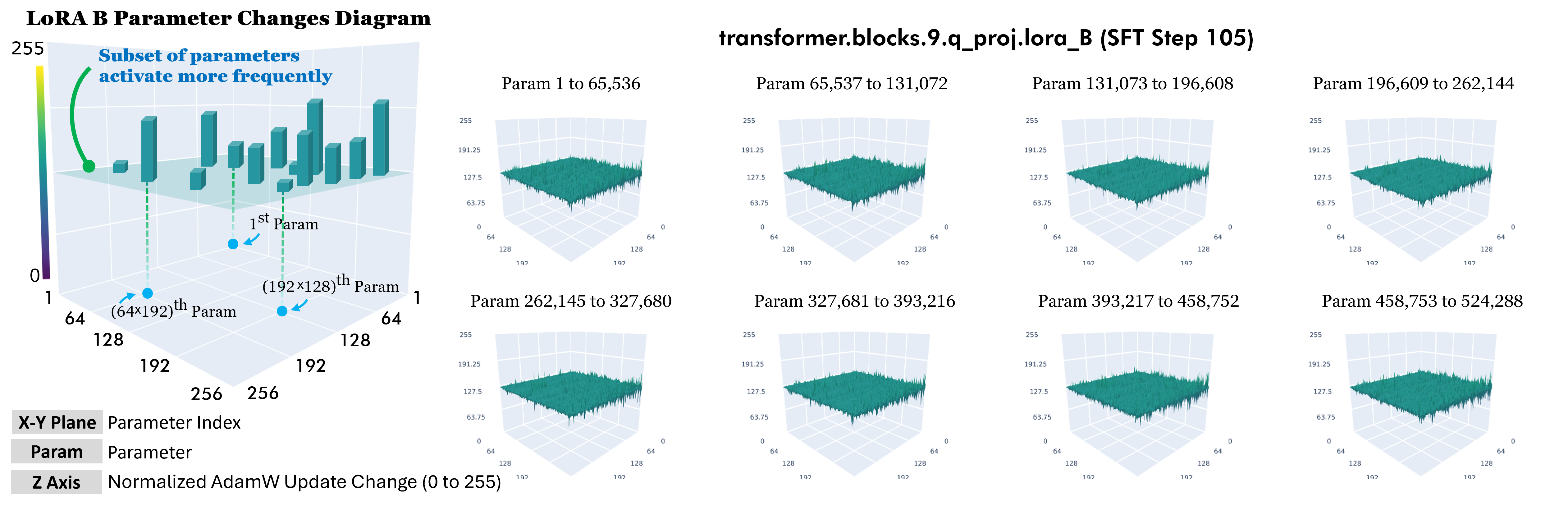}
\caption{Visualization of LoRA-B parameter updates at training step 105. A $4096\times128$ LoRA projection produces $524,288$ parameters, reshaped into a $256\times256\times8$ grid with the Z-axis showing normalized AdamW update magnitudes (scaled 0–255). Pronounced peaks indicate parameters critical for reasoning tasks, demonstrating that optimization dynamics create clear importance signatures.}
\label{fig:loraB_updates}
\end{figure}

\begin{figure}[!ht]
\centering
\includegraphics[width=0.99\textwidth]{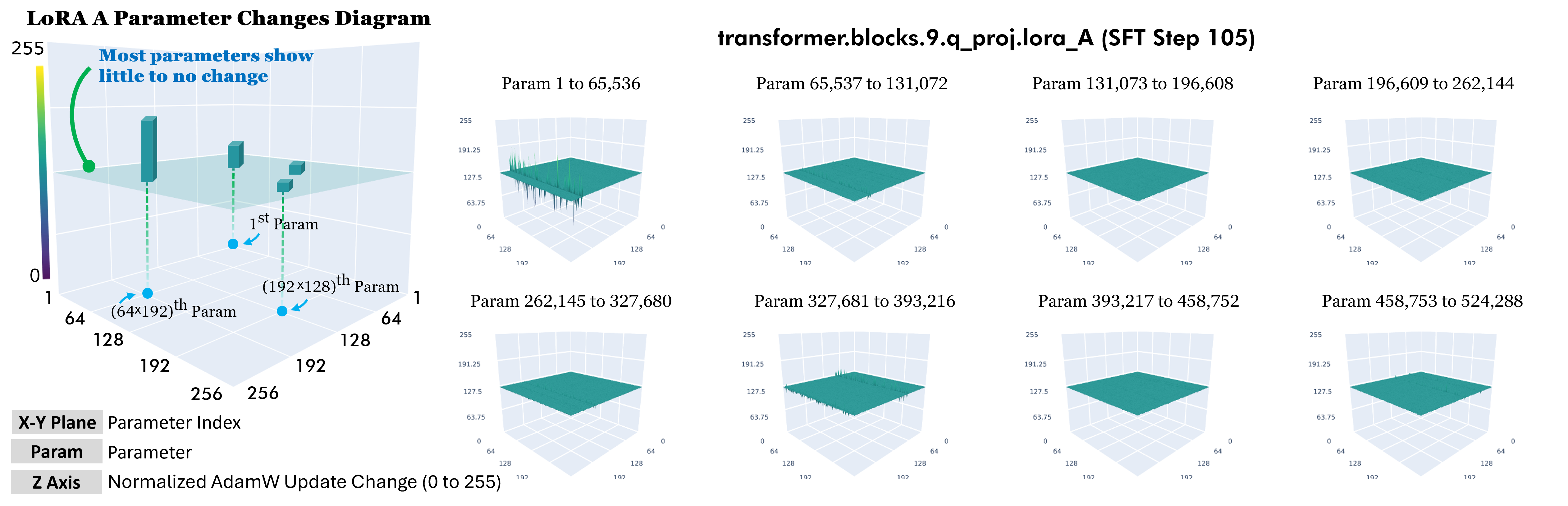}
\caption{Visualization of LoRA-A parameter updates at training step 105. A $128\times4096$ LoRA projection produces $524,288$ parameters, reshaped into a $256\times256\times8$ grid with the Z-axis showing normalized AdamW update magnitudes (scaled 0–255). Pronounced peaks indicate parameters critical for reasoning tasks, demonstrating that optimization dynamics create clear importance signatures.}
\label{fig:loraA_updates}
\end{figure}

\subsection{Module Selection Analysis}
\label{subsec:app_ablation}

To justify our choice of using the LoRA-B matrix from the Query projection, we conducted an ablation study comparing different modules and reduction strategies. Table \ref{tab:module_ablation} shows that the Query projection's LoRA-B matrix with row-wise energy reduction provides the most reliable stability signal.

\begin{table}[!ht]
\centering
\caption{Ablation study on module selection and reduction strategy. Results show average KL divergence stability (lower is better) across 100 validation examples from GSM8K.}
\label{tab:module_ablation}
\begin{tabular}{lcc}
\toprule
\textbf{Module} & \textbf{Row-wise Energy} & \textbf{Row-wise Mean} \\
\midrule
Query (LoRA-A) & 0.142 & 0.168 \\
Query (LoRA-B) & \textbf{0.089} & 0.103 \\
Key (LoRA-B) & 0.124 & 0.139 \\
Value (LoRA-B) & 0.117 & 0.128 \\
\bottomrule
\end{tabular}
\end{table}

\section{Theoretical Foundations of Early Termination in EDIT}
\label{app:guarantees}

\subsection{Setup and Notation}
\label{subsec:app_setup}

At denoising step $t$, let $\mathcal{S}_t$ denote the visible tokens, and $\mathcal{I}_t = \mathcal{S}_{t-1} \cap \mathcal{S}_t$ the matched support between steps $t-1$ and $t$. Let $\tilde{P}^{(t)}$ be the probability distribution on $\mathcal{I}_t$ obtained by restricting $P^{(t)}$ to $\mathcal{I}_t$ and renormalizing. Define the per-step divergence
\begin{equation}
D_t = D_{\text{KL}}(\tilde{P}^{(t)} \parallel \tilde{P}^{(t-1)}).
\end{equation}

EDIT declares stability when $D_{t-\Omega+1}, \ldots, D_t \leq \delta$ for some integers $\Omega \geq 1$ and threshold $\delta > 0$ (with fixed in-block temperature when forming $P^{(\cdot)}$).

We use the total variation distance $\text{TV}(p,q) = \frac{1}{2}\sum_{i}|p_i - q_i|$ and Pinsker's inequality $\text{TV}(p,q) \leq \sqrt{\frac{1}{2}D_{\text{KL}}(p \parallel q)}$.

\subsection{Multi-Step Control from Run-Length KL}
\label{subsec:app_multistep}

\begin{lemma}[Run-length KL implies multi-step TV bound]
\label{lem:runlength}
If $D_{t-\Omega+1}, \ldots, D_t \leq \delta$, then
\begin{equation}
\text{TV}(\tilde{P}^{(t)}, \tilde{P}^{(t-\Omega)}) \leq \sum_{r=t-\Omega+1}^{t} \text{TV}(\tilde{P}^{(r)}, \tilde{P}^{(r-1)}) \leq \Omega\sqrt{\frac{\delta}{2}}.
\end{equation}
\end{lemma}

\begin{proof}
The result follows from the triangle inequality for total variation distance, then applying Pinsker's inequality to each summand.
\end{proof}

\subsection{Local Argmax Invariance at the Stopping Time}
\label{subsec:app_argmax}

Let $i^*(t) = \arg\max_{s \in \mathcal{I}_t} \tilde{P}^{(t)}(s)$ and $m_t = \tilde{P}^{(t)}_{(1)} - \tilde{P}^{(t)}_{(2)}$ be the top-2 margin on $\mathcal{I}_t$.

\begin{theorem}[Local argmax invariance certificate]
\label{thm:local_argmax}
If $D_{t-\Omega+1}, \ldots, D_t \leq \delta$ and
\begin{equation}
\Omega\sqrt{\frac{\delta}{2}} < \frac{1}{2}m_t,
\end{equation}
then $i^*(t') = i^*(t)$ for all $t' \in \{t-\Omega, \ldots, t\}$.
\end{theorem}

\begin{proof}
If $i^*$ changed between $p = \tilde{P}^{(t)}$ and $q = \tilde{P}^{(t')}$, then $\text{TV}(p,q) \geq \frac{1}{2}(p_{(1)} - p_{(2)}) = \frac{1}{2}m_t$ (considering the mass that must move between the top two coordinates when the argmax changes). This contradicts Lemma~\ref{lem:runlength}.
\end{proof}

\textbf{Interpretation:} When EDIT stops and the inequality holds, the predicted token on the matched support has been unchanged for the past $\Omega$ steps—providing a verifiable certificate attached to the stopping decision.

\subsection{Future-Step Robustness via Contraction}
\label{subsec:app_future}

After the last unmasking in a block, let $K_r$ denote the Markov operator that maps $\tilde{P}^{(r-1)}$ to $\tilde{P}^{(r)}$ on the fixed support. Define the Dobrushin coefficient $\alpha(K_r) = \sup_{p \neq q} \frac{\text{TV}(pK_r, qK_r)}{\text{TV}(p,q)} \in [0,1]$.

\begin{assumption}[Local contraction post-unmasking]
\label{asmp:contraction}
There exists $\alpha < 1$ such that $\alpha(K_r) \leq \alpha$ for all $r \geq t+1$ within the block. This property is standard for convergent Markov chains and can be verified empirically on validation data.
\end{assumption}

\begin{theorem}[Tail movement bound and global argmax preservation]
\label{thm:tail}
Under Assumption~\ref{asmp:contraction},
\\
if
$D_{t-\Omega+1}, \ldots, D_t \leq \delta$,
then for any $s \geq 1$,
\begin{equation}
\text{TV}(\tilde{P}^{(t+s)}, \tilde{P}^{(t)}) \leq \frac{\alpha^s}{1-\alpha} \text{TV}(\tilde{P}^{(t)}, \tilde{P}^{(t-1)}) \leq \frac{\alpha^s}{1-\alpha}\sqrt{\frac{\delta}{2}},
\end{equation}
and thus $\sup_{s \geq 1} \text{TV}(\tilde{P}^{(t+s)}, \tilde{P}^{(t)}) \leq \frac{1}{1-\alpha}\sqrt{\frac{\delta}{2}}$.

If additionally
\begin{equation}
\Omega\sqrt{\frac{\delta}{2}} + \frac{1}{1-\alpha}\sqrt{\frac{\delta}{2}} < \frac{1}{2}m_t,
\end{equation}
then $i^*(t+s) = i^*(t)$ for all $s \geq 0$ (argmax is preserved forever on the fixed support).
\end{theorem}

\begin{proof}
One-step TV contracts by at most $\alpha$; summing the geometric tail yields the bound. The argmax preservation follows by the same margin argument as in Theorem~\ref{thm:local_argmax}.
\end{proof}

\subsection{Stability of Lipschitz Observables}
\label{subsec:app_lipschitz}

\begin{theorem}[Stability of Lipschitz functionals]
\label{thm:lipschitz}
Let $F: \Delta \to \mathbb{R}$ satisfy $|F(p) - F(q)| \leq L \cdot \text{TV}(p,q)$ for all $p,q$. If $D_{t-\Omega+1}, \ldots, D_t \leq \delta$, then
\begin{equation}
|F(\tilde{P}^{(t)}) - F(\tilde{P}^{(t-\Omega)})| \leq L \cdot \Omega\sqrt{\frac{\delta}{2}}.
\end{equation}
Under Assumption~\ref{asmp:contraction},
\begin{equation}
\sup_{s \geq 1}|F(\tilde{P}^{(t+s)}) - F(\tilde{P}^{(t)})| \leq \frac{L}{1-\alpha}\sqrt{\frac{\delta}{2}}.
\end{equation}
\end{theorem}

\begin{proof}
Direct application of Lemma~\ref{lem:runlength} and Theorem~\ref{thm:tail} with the Lipschitz property.
\end{proof}

\subsection{Practical Calibration of $(\delta, \Omega)$}
\label{subsec:app_calibration}

Let $M$ denote the top-2 margin at EDIT's stopping time on a validation set, and $q_{1-\beta}$ its $(1-\beta)$-quantile. Estimate a post-unmasking contraction bound by
\begin{equation}
\hat{\alpha} = \max_{\text{val instances, late } r} \frac{\text{TV}(\tilde{P}^{(r+1)}, \tilde{P}^{(r)})}{\text{TV}(\tilde{P}^{(r)}, \tilde{P}^{(r-1)})}.
\end{equation}

\begin{corollary}[PAC-style guarantee for the final answer]
\label{cor:pac}
Choose $(\delta, \Omega)$ to satisfy
\begin{equation}
\Omega\sqrt{\frac{\delta}{2}} + \frac{1}{1-\hat{\alpha}}\sqrt{\frac{\delta}{2}} \leq \frac{1}{2}q_{1-\beta}.
\end{equation}
Then, with probability at least $1-\beta$ over test instances, the top-1 token at EDIT's stopping time equals the top-1 token obtained by continuing denoising indefinitely (on the fixed support).
\end{corollary}

\begin{proof}
Apply Theorem~\ref{thm:tail} and the definition of $q_{1-\beta}$.
\end{proof}

\textbf{Reporting recommendation:} Alongside accuracy and step reductions, report the fraction of test instances that satisfy Corollary~\ref{cor:pac} ("percentage of certified stops"). This quantifies how often EDIT halts with a provable correctness certificate.

\section{Token-Wise EDIT: Per-Token Freezing with Certificates}
\label{app:token}

\subsection{Local Stability Statistics and Rule}
\label{subsec:app_token_stats}

Let $U \in \mathbb{R}^{d \times k}$ denote a fixed reasoning subspace (construction examples provided below). For each visible token $s$ at step $t$, define its subspace coordinates $g_s^{(t)} = U^\top f_s^{(t)} \in \mathbb{R}^k$ and the local distribution
\begin{equation}
Q_s^{(t)}(j) = \frac{\exp(|g_{s,j}^{(t)}|/\tau_{\text{sub}})}{\sum_{\ell=1}^k \exp(|g_{s,\ell}^{(t)}|/\tau_{\text{sub}})}, \quad j = 1, \ldots, k,
\end{equation}
with fixed $\tau_{\text{sub}} > 0$.

Define the per-token KL $D_{s,t} = D_{\text{KL}}(Q_s^{(t)} \parallel Q_s^{(t-1)})$ and the run-length condition: token $s$ is locally stable at $t$ if $D_{s,t-r} \leq \delta_{\text{tok}}$ for $r = 0, \ldots, \Omega_{\text{tok}} - 1$. The token-wise EDIT freezes $s$ at $t$ (setting $f_s^{(t')} \equiv f_s^{(t)}$ for all $t' > t$) whenever this condition holds.

\subsection{Per-Token Certificates}
\label{subsec:app_token_certs}

\begin{lemma}[Local run-length bound]
\label{lem:tok_run}
If $D_{s,t-r} \leq \delta_{\text{tok}}$ for $r = 0, \ldots, \Omega_{\text{tok}} - 1$, then
\begin{equation}
\text{TV}(Q_s^{(t)}, Q_s^{(t-\Omega_{\text{tok}})}) \leq \Omega_{\text{tok}}\sqrt{\frac{\delta_{\text{tok}}}{2}}.
\end{equation}
\end{lemma}

\begin{proof}
Triangle inequality and Pinsker's inequality, as in Lemma~\ref{lem:runlength}.
\end{proof}

Let $j_s^*(t) = \arg\max_j Q_s^{(t)}(j)$ and $m_s(t) = Q_{s,(1)}^{(t)} - Q_{s,(2)}^{(t)}$ be the local top-2 margin.

\begin{theorem}[Dominant subspace-component invariance per token]
\label{thm:tok_inv}
If the condition of Lemma~\ref{lem:tok_run} holds and $\Omega_{\text{tok}}\sqrt{\delta_{\text{tok}}/2} < \frac{1}{2}m_s(t)$, then $j_s^*(t') = j_s^*(t)$ for all $t' \in \{t - \Omega_{\text{tok}}, \ldots, t\}$.
\end{theorem}

\begin{proof}
Identical to Theorem~\ref{thm:local_argmax} but applied to $Q_s^{(\cdot)}$.
\end{proof}

\subsection{Freezing Safety Under Weak Coupling}
\label{subsec:app_token_safety}

We quantify how freezing one token perturbs the global distribution $P^{(t)}$.

\begin{assumption}[Weak cross-token coupling]
\label{asmp:beta}
There exists $\beta_s \geq 0$ such that, if two states at step $r$ differ only in token $s$ by $\Delta_s$ (that is, $f_s^{(r)} \mapsto f_s^{(r)} + \Delta_s$), then their next-step global distributions satisfy
\begin{equation}
\text{TV}(\tilde{P}^{(r+1)}, \tilde{P}'^{(r+1)}) \leq \beta_s \|\Delta_s\|_2.
\end{equation}
This $\beta_s$ can be estimated on validation by finite-difference probes.
\end{assumption}

\begin{assumption}[Post-unmasking contraction]
\label{asmp:alpha}
Within a block after the last unmasking, the Markov operators contract TV with coefficient $\alpha < 1$ as in Assumption~\ref{asmp:contraction}.
\end{assumption}

\begin{theorem}[Safety of freezing token $s$]
\label{thm:freeze_safe}
Suppose token $s$ satisfies the local stability condition at time $t$, and set
\begin{equation}
\varepsilon_s = \max_{r \in \{t-\Omega_{\text{tok}}+1, \ldots, t\}} \|f_s^{(r)} - f_s^{(r-1)}\|_2.
\end{equation}
If token $s$ is frozen at $t$, then for all $u \geq 1$,
\begin{equation}
\text{TV}(\tilde{P}_{\text{frozen}}^{(t+u)}, \tilde{P}_{\text{unfrozen}}^{(t+u)}) \leq \frac{\beta_s}{1-\alpha}\varepsilon_s.
\end{equation}
Consequently, if
\begin{equation}
\Omega\sqrt{\frac{\delta}{2}} + \frac{\beta_s}{1-\alpha}\varepsilon_s < \frac{1}{2}m_t,
\end{equation}
where $m_t$ is the global top-2 margin at $t$, then the global argmax remains unchanged forever (on the fixed support) after freezing token $s$.
\end{theorem}

\begin{proof}
One-step deviation is at most $\beta_s\varepsilon_s$ by Assumption~\ref{asmp:beta}. Propagating under Assumption~\ref{asmp:alpha} yields a geometric tail bound $\sum_{j \geq 0} \alpha^j \beta_s\varepsilon_s = \frac{\beta_s}{1-\alpha}\varepsilon_s$. Combine with Theorem~\ref{thm:local_argmax}.
\end{proof}

\textbf{Construction of $U$ and practical tuning:} A simple choice is to take $\bar{U}_B$ from Equation~\ref{eq:adamw_evolution_tensor}, compute its left singular vectors, and set $U$ to the top $k \in \{2, 3, 4\}$ vectors. Empirically, $k = 3$ or $k = 4$ provides good stability-efficiency trade-offs. On validation, choose $(\delta_{\text{tok}}, \Omega_{\text{tok}})$ to maximize frozen-token count subject to Theorem~\ref{thm:tok_inv}'s margin condition.

\section{Subspace EDIT: Replacing the Reasoning Vector by a Subspace}
\label{app:subspace}

\subsection{Definition}
\label{subsec:app_subspace_def}

Let $U \in \mathbb{R}^{d \times k}$ with orthonormal columns ($k \geq 1$). Replace the scalar alignment in Equation~\ref{eq:similarity} by a subspace score:
\begin{equation}
\text{Sim}_s^{(t)} = \|U^\top f_s^{(t)}\|_2 \quad \text{or} \quad \text{Sim}_s^{(t)} = \frac{\|U^\top f_s^{(t)}\|_2}{\|f_s^{(t)}\|_2} \quad \text{(subspace cosine)},
\end{equation}
and form $P^{(t)}$ from Equation~\ref{eq:distribution} with the same fixed in-block temperature $\tau_{\text{blk}}$. All other components of EDIT (matched-support renormalization, KL divergence, run-length rule) remain unchanged.

\subsection{Inherited Guarantees}
\label{subsec:app_subspace_guarantees}

\begin{proposition}[Guarantees are shape-agnostic in the similarity]
\label{prop:subspace}
The statements and proofs of Lemma~\ref{lem:runlength}, Theorems~\ref{thm:local_argmax}–\ref{thm:lipschitz}, and Corollary~\ref{cor:pac} hold verbatim under the subspace similarity above.
\end{proposition}

\begin{proof}
The guarantees depend only on the distributions $\tilde{P}^{(\cdot)}$ and their KL/TV relations. The construction of $\text{Sim}_s^{(t)}$ enters only through $P^{(t)}$'s definition with a fixed temperature.
\end{proof}

\textbf{Constructing $U$:} Two practical choices are available. First, perform SVD of $\bar{U}_B$ (Equation~\ref{eq:adamw_evolution_tensor}) and retain the top $k$ left singular vectors. Second, use CCA between step-averaged preconditioned gradients and activations $\{f_s\}$ in the chosen module. Ablations suggest small $k$ values ($2$–$4$) are sufficient and can stabilize earlier than single-vector approaches, potentially offering improved efficiency-accuracy trade-offs.

\section{EDIT Algorithm}
\label{app:edit_algo}

Algorithm~\ref{alg:edit} presents the full EDIT procedure, while Figure~\ref{fig:EDIT_diagram} provides a visual illustration of its workflow.

\begin{algorithm}
\caption{EDIT: Early Diffusion Inference Termination}
\label{alg:edit}
\begin{algorithmic}[1]
\REQUIRE Input sequence; thresholds $\delta$, $\Omega$; block temperature $\tau_{\text{blk}}$; fine-tuning steps $\mathcal{K}$
\ENSURE Generated text with adaptive early termination
\STATE \textbf{// Phase 1: Training-Time Metadata Extraction (one-time)}
\FOR{each fine-tuning step $k = 1$ to $\mathcal{K}$}
    \STATE Update AdamW moments $M_{k,B}$, $V_{k,B}$ using Eq.~\ref{eq:adamw_moments}
    \STATE Compute update tensor $U_{k,B}$ using Eq.~\ref{eq:adamw_update}
\ENDFOR
\STATE Compute AdamW evolution tensor $\bar{U}_B$ using Eq.~\ref{eq:adamw_evolution_tensor}
\STATE Reduce to feature vector $\mathbf{u}$ using Eq.~\ref{eq:reduction}
\STATE Store $\mathbf{u}$ as metadata for inference use
\STATE \textbf{// Phase 2: Inference-Time Early Termination (uses precomputed $\mathbf{u}$)}
\FOR{each diffusion block $b = 1$ to $B$}
    \STATE Initialize visible set $\mathcal{S}_1$ by unmasking schedule
    \STATE Compute activations $\mathbf{f}_s^{(1)}$ for $s \in \mathcal{S}_1$
    \STATE Compute initial distribution $P^{(1)}$ using Eq.~\ref{eq:similarity} and \ref{eq:distribution}
    \STATE Set stability counter $c \leftarrow 0$
    \FOR{denoising step $t = 2$ to $T_b$}
        \STATE Update visible set $\mathcal{S}_t$ according to unmasking schedule
        \STATE Compute activations $\mathbf{f}_s^{(t)}$ for $s \in \mathcal{S}_t$
        \STATE Compute distribution $P^{(t)}$ using Eq.~\ref{eq:similarity} and \ref{eq:distribution}
        \STATE Set intersection $\mathcal{I}_t = \mathcal{S}_{t-1} \cap \mathcal{S}_t$
        \STATE Renormalize to $\tilde{P}^{(t)}$, $\tilde{P}^{(t-1)}$ using Eq.~\ref{eq:renorm}
        \STATE Compute $D_t = D_{\text{KL}}(\tilde{P}^{(t)} \parallel \tilde{P}^{(t-1)})$ using Eq.~\ref{eq:kl_divergence}
        \IF{$D_t < \delta$}
            \STATE $c \leftarrow c + 1$
        \ELSE
            \STATE $c \leftarrow 0$
        \ENDIF
        \IF{$c \geq \Omega$}
            \STATE \textbf{break} // Early termination for block $b$
        \ENDIF
    \ENDFOR
\ENDFOR
\end{algorithmic}
\end{algorithm}

\begin{figure}[!t]
\centering
\includegraphics[width=0.99\textwidth]{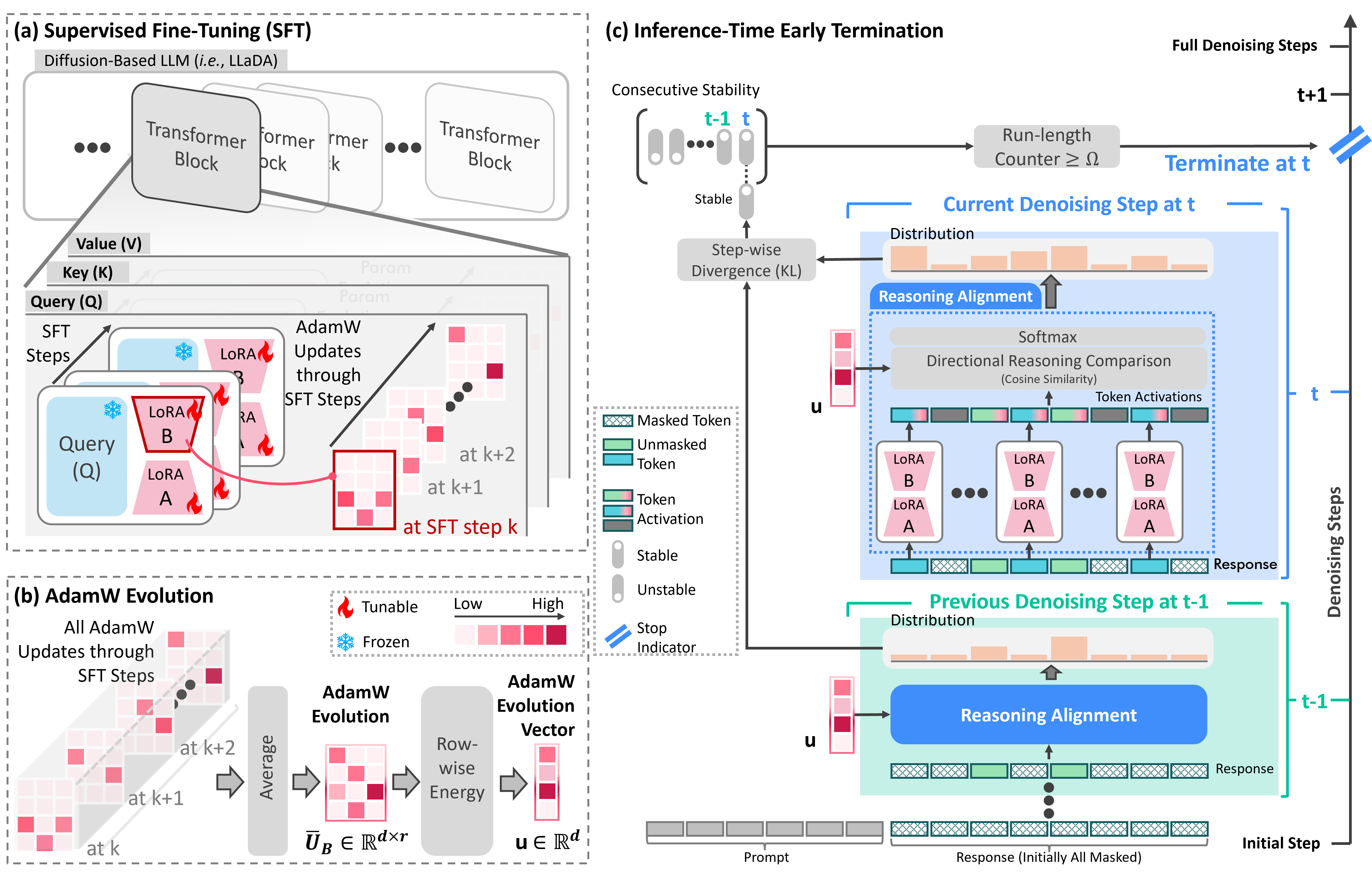}
\caption{
Overview of EDIT.
\textit{(a) Supervised Fine-Tuning (SFT)}: AdamW moment estimates track LoRA-B updates across steps, where some parameters consistently receive strong, directionally-aligned updates that encode reasoning patterns.
\textit{(b) AdamW Evolution}: Aggregating these updates yields a compact evolution vector $\mathbf{u}$ that encodes reasoning-relevant parameter importance.
\textit{(c) Inference-Time Early Termination}: At inference, token activations are compared with the preserved evolution vector $\mathbf{u}$; reasoning alignment is monitored via cosine similarity and KL divergence across steps. Once stability persists for consecutive steps, termination occurs before full denoising, reducing cost without loss of quality.
}
\label{fig:EDIT_diagram}
\end{figure}

\section{Gradient-Based Justification for Early Termination}
\label{app:gradview_edit}

To determine when denoising steps can be truncated safely, we compare inference pseudo-gradients with the SFT gradients $G_{k,B}$ on LoRA-B layers. During inference, at each denoising step $t \in \{1,\dots,T_b\}$, the model produces logits $z_t(s)$ for every token position $s$ in the block of length $L$.
Since no ground-truth labels are available at inference time, the only informative signal comes from the evolution of predictive distributions across steps. We denote $p_\theta(z_t(s))$ as the model’s prediction for token $s$ at step $t$ and $p_\theta(z_{t+1}(s))$ as the refined prediction at step $t{+}1$. The KL divergence between them measures how much the model’s belief changes across steps, with larger values indicating ongoing refinement and smaller values indicating stabilization.
We therefore define the pseudo-gradient as
\begin{equation}
\tilde{G}_{t,B}
\;=\;
\nabla_{B}
\sum_{s \in S_{t+1}}
\mathrm{KL}\left(
p_\theta\big(z_t(s)\big)
\,\big|\,
p_\theta\big(z_{t+1}(s)\big)
\right),
\end{equation}
where $S_{t+1}$ denotes the visible token set at step $t+1$.
We compute the pseudo-gradient by evaluating the KL divergence between consecutive predictive distributions at steps $t$ and $t+1$, restricted to $S_{t+1}$ so that only effective (unmasked) tokens contribute. Backpropagating this divergence through the LoRA-B parameters yields $\tilde{G}_{t,B}$, and we record its root-mean-square (RMS) magnitude as a scalar summary for step $t$. Repeating this across all denoising steps produces a trajectory of pseudo-gradients that characterizes the sensitivity of inference dynamics.

For training, we take the gradients $G_{k,B}$ observed at each SFT step $k$, compute their RMS magnitudes, and summarize them by a mean $\mu_{\text{SFT}}$ and a variance band. These gradients fluctuate around the mean within a bounded band, defining the stable regime in which the model was optimized. By overlaying the inference pseudo-gradients $\tilde{G}_{t,B}$ with this SFT reference, we obtain a principled test of alignment. 
Convergence is declared when pseudo-gradients (1) approach $\mu_{\text{SFT}}$ and (2) remain within this band, beyond which further denoising adds cost without benefit.
Initially, $\tilde{G}_{t,B}$ deviates from the SFT regime, but after several iterations it reaches a convergence point $t_{\text{conv}} = \arg\min_t |\mathrm{RMS}(\tilde{G}_{t,B}) - \mu_{\text{SFT}}|$, after which the pseudo-gradients oscillate around the SFT mean $\mu_{\text{SFT}}$ in a manner statistically consistent with $G_{k,B}$. This indicates that inference has entered the same training-consistent regime, and further denoising steps add computation without providing additional alignment benefit.

Empirically, on the GPQA benchmark (sequence length 128) we analyze the second diffusion block and observe in Figure \ref{fig:app_grad_analysis} that the pseudo-gradients reach a convergence point (marked by the yellow $\color{DarkYellow}{\blacktriangledown}$) at the 19-th denoising step. Beyond this point, they fluctuate stably around the SFT mean, indicating entry into the training-consistent regime. Terminating at $\sim$20 steps per block therefore preserves fidelity while reducing computation, consistent with Table \ref{tab:efficiency}, which shows an average of 40.3 steps for two diffusion blocks (\ie, $\sim$20 steps each) on GPQA, and with Table~\ref{tab:accuracy}, which confirms that accuracy remains competitive.

After the convergence point at step 19, the pseudo-gradients stay within the SFT variance band, oscillating around the mean, but a spike appears near step 25. This behavior may reflect the dynamics of late denoising. At this stage, most unmasked (visible) tokens have stabilized, while updates focus on the remaining masked positions. These late updates often involve low-information elements such as function words, punctuation, or minor phrasing, which carry little semantic weight but can still trigger abrupt shifts in the predictive distribution. Importantly, because these refinements concern filler-like tokens, they do not compromise the earlier convergence, which already ensures fidelity comparable to full denoising.

\vspace{15pt}
\begin{figure}[h]
\centering
\includegraphics[width=0.99\textwidth]{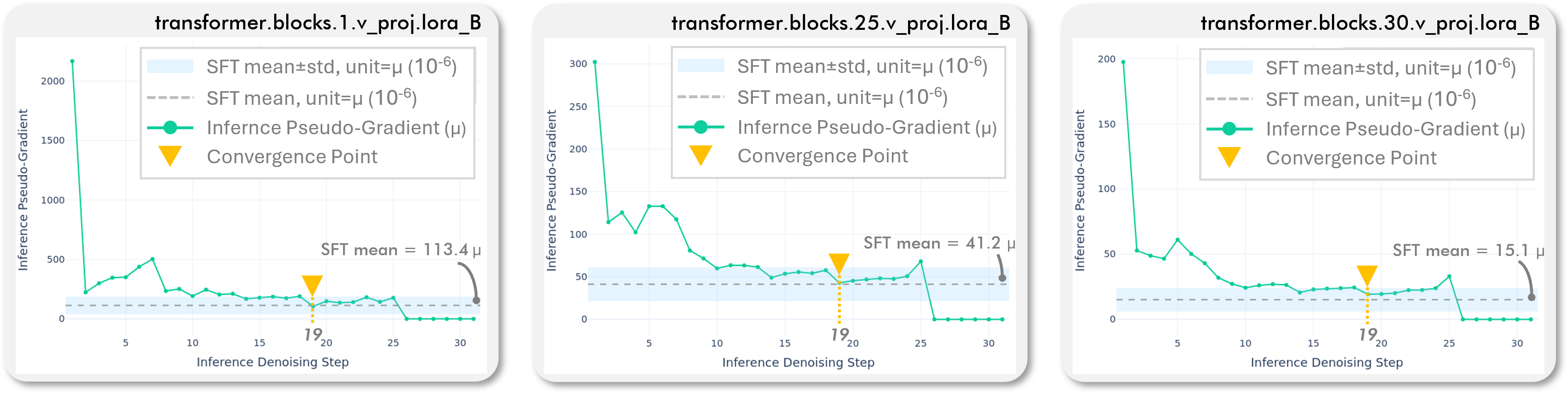}
\caption{
Gradient-based analysis of training–inference alignment on GPQA (sequence length 128, second diffusion block). The curve shows the RMS pseudo-gradients $\tilde{G}_{t,B}$ across denoising steps, compared against the SFT gradient mean (dashed) and variance band (shaded). The convergence point (yellow \textcolor{DarkYellow}{$\blacktriangledown$}) occurs at the 19-th step, after which pseudo-gradients oscillate stably around the SFT mean. This alignment indicates that terminating at $\sim$20 steps per block maintains fidelity comparable to full denoising while reducing computation, consistent with Table \ref{tab:efficiency} (40.3 steps for two blocks).
}
\label{fig:app_grad_analysis}
\end{figure}

\vspace{15pt}

We also show additional examples in Figure \ref{fig:grad_analysis_countdown} to \ref{fig:grad_analysis_gsm8k}, which presents Countdown, Sudoku, MATH500, and GSM8K. Their convergence points occur at the 20-th, 18–19-th, 19-th, and 23-rd steps respectively. In each task, terminating around these points preserves fidelity while reducing computation, consistent with the average step counts reported in Table \ref{tab:efficiency} (40.4, 38.3, 38.1, and 42.8 steps for two blocks). These results demonstrate that pseudo-gradient convergence provides a consistent signal of reasoning step completion across diverse benchmarks.

\begin{figure}[h]
\centering
\includegraphics[width=\textwidth]{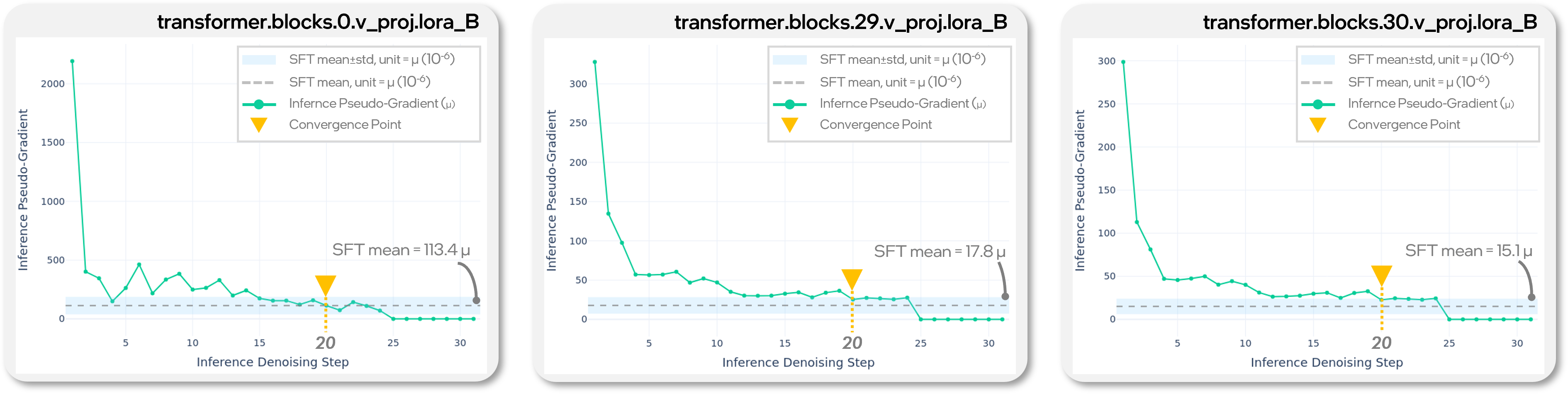}
\caption{Countdown (sequence length 128, second diffusion block). The convergence point (yellow \textcolor{DarkYellow}{$\blacktriangledown$}) occurs at the 20-th step. Terminating at $\sim$20 steps per block maintains fidelity comparable to full denoising while reducing computation, consistent with Table \ref{tab:efficiency} (40.4 steps for two blocks).}
\label{fig:grad_analysis_countdown}
\end{figure}

\begin{figure}[h]
\centering
\includegraphics[width=\textwidth]{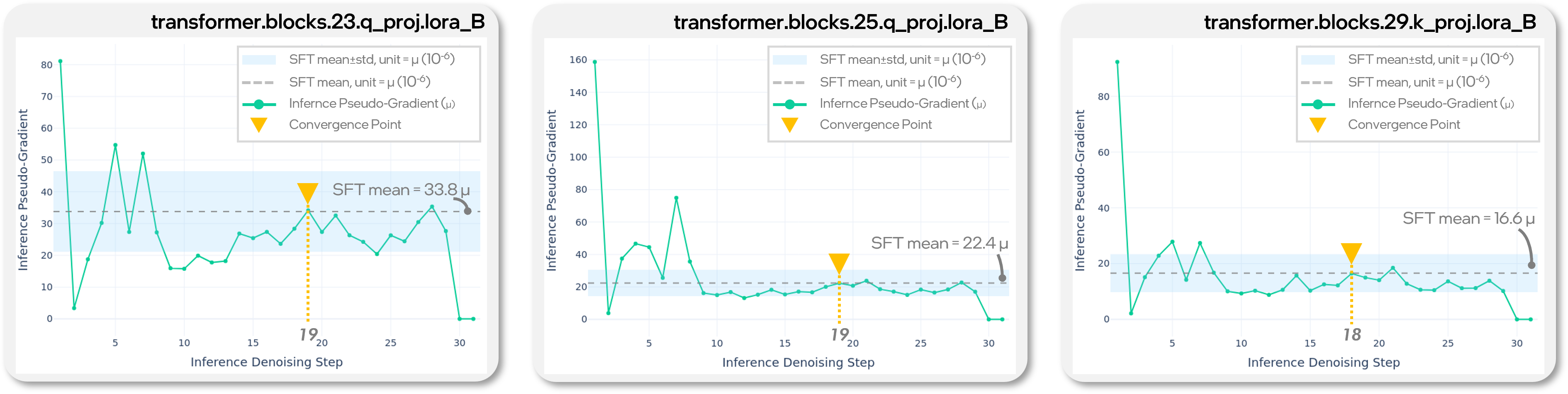}
\caption{Sudoku (sequence length 128, second diffusion block). The convergence point (yellow \textcolor{DarkYellow}{$\blacktriangledown$}) occurs at the 18-th and 19-th steps. Terminating at $\sim$19 steps per block maintains fidelity comparable to full denoising while reducing computation, consistent with Table \ref{tab:efficiency} (38.3 steps for two blocks).}
\label{fig:grad_analysis_sudoku}
\end{figure}

\begin{figure}[h]
\centering
\includegraphics[width=\textwidth]{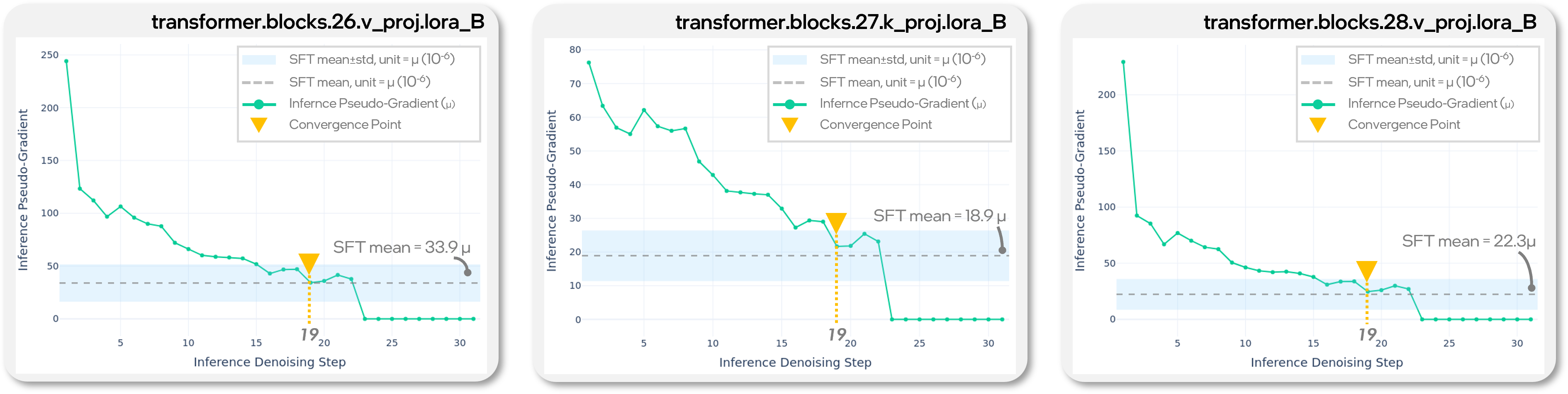}
\caption{MATH500 (sequence length 128, second diffusion block). The convergence point (yellow \textcolor{DarkYellow}{$\blacktriangledown$}) occurs at the 19-th step. Terminating at $\sim$19 steps per block maintains fidelity comparable to full denoising while reducing computation, consistent with Table \ref{tab:efficiency} (38.1 steps for two blocks).}
\label{fig:grad_analysis_math}
\end{figure}

\begin{figure}[h]
\centering
\includegraphics[width=\textwidth]{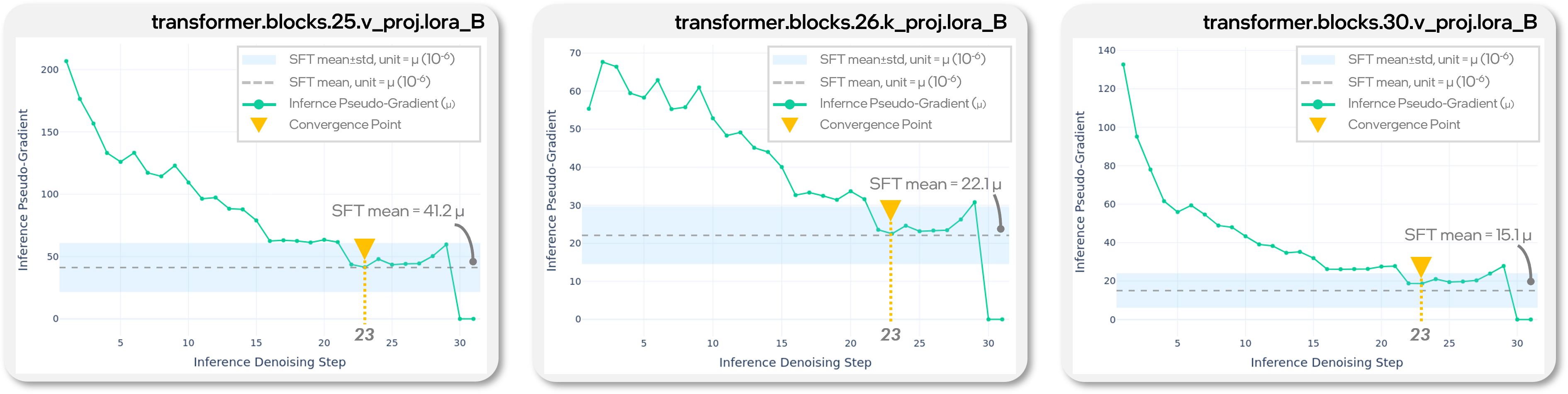}
\caption{GSM8K (sequence length 128, second diffusion block). The convergence point (yellow \textcolor{DarkYellow}{$\blacktriangledown$}) occurs at the 23-rd step. Terminating at $\sim$22 steps per block maintains fidelity comparable to full denoising while reducing computation, consistent with Table \ref{tab:efficiency} (42.8 steps for two blocks).}
\label{fig:grad_analysis_gsm8k}
\end{figure}

\end{document}